\documentclass[11pt]{article}
\usepackage{style}

\usepackage{amsthm ,amsmath, amssymb, natbib, graphicx, url, algorithm2e}

\usepackage{times}

\author{
Yukti Makhija$^{\dagger}$ \enspace
Rishi Saket$^{\dagger}$  \vspace{1mm} \\
Google Research \\
\texttt{\{yuktimakhija,rishisaket\}@google.com}
}
\date{}
\usepackage{booktabs}       %
\usepackage{amsfonts}       %
\usepackage{nicefrac}       %
\usepackage{microtype}      %
\usepackage{mathtools,amsfonts,amssymb,mdframed,xspace,xfrac,bm}
\usepackage{makecell}
\usepackage{amssymb,amsopn,mathtools}
\usepackage{mathrsfs}
\usepackage{complexity}
\usepackage{enumitem}
\usepackage{subcaption}
\usepackage{framed}

\usepackage{diagbox}
\usepackage{multirow}
\usepackage{booktabs}

\newtheorem{theorem}{Theorem}[section]

\newtheorem{lemma}[theorem]{Lemma}

\newcommand{\eps}{\varepsilon}

\renewcommand{\E}{\mathbb{E}}

\newcommand{\pos}{{\sf{pos}}}
\renewcommand{\R}{\mathbb{R}}

\newcommand{\mc}[1]{\ensuremath{\mathcal{#1}}}

\newcommand{\tn}[1]{\ensuremath{\textnormal{#1}}}
\newcommand{\ol}[1]{\ensuremath{\overline{#1}}}

\newcommand{\bu}{{\mathbf{u}}}
\newcommand{\bv}{{\mathbf{v}}}

\newcommand{\br}{{\mathbf{r}}}

\title{Weak to Strong Learning from Aggregate Labels}

\begin{document}
\def\thefootnote{$\dagger$}\footnotetext{currently at Google DeepMind.} 
\maketitle

\begin{abstract}%
  In learning from aggregate labels, the training data consists of sets or ``bags'' of feature-vectors (instances) along with an aggregate label for each bag derived from the (usually $\{0,1\}$-valued) labels of its constituent instances. In the \textit{learning from label proportions} (LLP) setting, the aggregate label of a bag is the average of the instance labels, whereas in \textit{multiple instance learning} (MIL) it is the OR. The goal is to train an instance-level predictor, which is typically achieved by fitting a model on the training data, in particular one that maximizes the accuracy which is the fraction of \textit{satisfied} bags i.e., those on which the model's induced labels are consistent with the target aggregate label. A weak learner in this context is one which has at a constant accuracy $ < 1$ on the training bags, while a strong learner's accuracy can be arbitrarily close to $1$. We study the problem of using a weak learner on such training bags with aggregate labels to obtain a strong learner, analogous to supervised learning for which boosting algorithms are known. Our first result shows the impossibility of boosting in the LLP setting using weak classifiers of any accuracy $< 1$ by constructing a collection of bags for which such weak learners (for any weight assignment) exist, while not admitting any strong learner. A variant of this construction also rules out boosting in MIL for a non-trivial range of weak learner accuracy. \\
  In the LLP setting however, we show that a weak learner (with small accuracy) on large enough bags can in fact be used to obtain a strong learner for small bags, in polynomial time. We also provide more efficient, sampling based variant of our procedure with probabilistic guarantees which are empirically validated on three real and two synthetic datasets. Our work is the first to theoretically study weak to strong learning from aggregate labels, with an algorithm to achieve the same for LLP, while proving the impossibility of boosting for both LLP and MIL.
\end{abstract}

\section{Introduction} \label{sec:intro}

In traditional, fully supervised learning, the training data consists of a collection of labeled feature-vectors (i.e., training examples) $\{(\bx_i \in \bm{\mc{X}}, y_i = y(\bx_i))\}_{i=1}^n$, for some domain $\bm{\mc{X}}$ where the mapping $y$ provides the feature-vector labels. In this paper we will consider the binary setting i.e., the labels are $\{0,1\}$-valued. %
The usual training goal is to find a good classifier $f : \bm{\mc{X}} \to \{0,1\}$ which maximizes the training accuracy $\left|\{i : f(\bx_i) = y_i\}\right|/n$. 
In recent times however, due to privacy~\citep{R10} or feasibility~\citep{CHR} constraints, in many applications the training label for each  training example is not available. Instead, the training data consists of sets or \emph{bags} of feature-vectors along with only the \emph{average} or equivalently \emph{sum} of the labels for each bag since bag size is known. This is called \emph{learning from label proportions} (LLP) in which the training set consists of labeled bags $\{(B_j, \ol{y}_j\}_{j=1}^m$ where  $B_j \subseteq \bm{\mc{X}}$ and $\ol{y}_j = \sum_{\bx \in B_j}y(\bx)$. The training goal is to fit a good classifier $f: \bm{\mc{X}} \to \{0,1\}$ on this bag-level training data. A related problem is \emph{multiple instance learning} (MIL) in which the label for each bag is the {\sf OR} of the boolean labels of its constituent feature vectors, while the goal of fitting a good feature-vector classifier remains the same. %
A natural metric for the goodness of fit in the LLP setting is to maximize the bag-level accuracy i.e., the fraction of \textit{satisfied} training bags, where a bag $(B, \ol{y})$ is satisfied if $\ol{y} = \left(\sum_{\bx \in B}f(\bx)\right)$. An analogous notion of accuracy for MIL is if $\ol{y} = \left(\bigvee_{\bx \in B}f(\bx)\right)$. Recent works~\citep{Saket21,Saket22} have studied the the computational learning aspect of LLP and MIL, and in particular showed that the problem of finding classifiers (even in the realizable case) of high bag-level accuracy can be NP-hard.

In supervised classification, \emph{boosting} (see \citep{AdaBoost,FSBook}) is a well known meta-technique which, given a training dataset uses an ensemble (typically a majority) of  \textit{weak} classifiers (on reweighed data) to output a hypothesis which has accuracy arbitrarily close to $1$ i.e., a \textit{strong} classifier. In the $\{0,1\}$-labels case a weak classifier has accuracy at least $(1/2 + \eps)$ for some $\eps > 0$, while that for a strong classifier is $(1 - \nu)$ where $\nu$ can be made arbitrarily small. Note that the threshold of $1/2$ for weak classification is the expected accuracy of random prediction on the training set.

To address the algorithmic learning problems in LLP and MIL, one could hope to apply boosting techniques to LLP and MIL settings as well. Here, we can define a weak classifier having some constant accuracy on the bags, while the notion of a strong classifier remains the same: that with an arbitrarily high accuracy. A natural question to ask is:

\textit{is there a way to do boosting using weak-classifiers to obtain a strong classifier in learning from aggregate labels?}

In this work we show that the above is \textit{impossible} even on $2$-sized bags for (i) LLP using weak classifiers of any accuracy $ < 1$, and (ii) for MIL using weak classifiers of any accuracy $< 2/3$. Specifically, we construct a collection of bags such that any reweighing of the bags admits a weak classifier of the desired accuracy while the original collection does not have admit \emph{any} strong classifier i.e., any labeling to the underlying feature vectors satisfies at most some constant $< 1$ fraction of the bags. We note that on bags of size $2$, for both LLP and MIL the worst-case accuracy obtained by using the random or any constant-valued classifier (all $0$s or all $1$s), is $1/2$. So, even for MIL we rule out boosting using weak classifiers with non-trivial accuracy in $[1/2, 2/3)$.  Our impossibility of boosting stands in contrast to previous work (e.g. \citep{MILBoost,AdaBoostLLP}) which empirically evaluate boosting heuristics for LLP and MIL -- our results are the first to show that such algorithms cannot provably yield a strong classifier.

While the above impossibility results are applicable to the boosting framework, one can ask:

\emph{is there some other way to derive a strong classifier from weak classifiers?}

Our next result answers this question in the affirmative for LLP: a weak classifier (of any constant accuracy $\gamma > 0$) on large bags can be used to derive a strong classifier on a training set of (smaller) bags. These large bags are each a union of $t$ training bags, where $t$ depends only on $\gamma$ and the desired accuracy of the strong classifier. While on $m$ training bags, the number of ($\approx m^t$) unions are polynomial-time for constant $t$, we also provide a significantly more efficient sampling version of this approach which provides the same guarantees with high probability. These are to the best of our knowledge the first methods obtaining strong classifiers from weak classifiers for LLP. For MIL on the other hand the question of such weak to strong learning remains open.

\subsection{Previous Related Work}
{\bf Multiple Instance Learning (MIL).} The study by \citet{DLL97} introduced MIL for drug activity detection, where the bag label is modeled as an {\sf OR} of its (unknown) instance labels, all labels are $\{0,1\}$-valued. The goal, given such a dataset of bags, is to train a classifier for instance labels. Theoretically, \citet{blum1998note} proved that noise tolerant PAC learnability implies MIL PAC learnability for iid bags, and generalization bounds for the classification error on bags were provided by \citet{ST12}. %
Methods including logistic regression, maximum likelihood and boosting with differentiable approximations to the {\sf OR} function~\citep{RC05, ramon2000multi,ZPV05} have been proposed. Diverse-density (DD) method~\citep{ML97} and its EM-based variant, EM-DD~\citep{ZG01} are specialised MIL techniques.
 Over the years this approach has found many applications in numerous areas, including drug discovery~\citep{ML97}, analysis of videos ~\citep{SDB13}, medical images~\citep{WYY15}, time series ~\citep{M98} and information retrieval~\citep{LY00}.

\noindent
 {\bf Learning from Label Proportions (LLP).} A variety of specialized LLP methods have been introduced till date: \citet{FK05} and \citet{HIL13} developed  MCMC techniques, \citet{MCO07} adapted traditional supervised learning techniques like $k$-NN and SVM, while clustering based methods were proposed by \citet{CLQZ09} and \citet{SM11}. Further, \citet{QSCL09} and \citet{PNCR14} devised specialized learning algorithms using bag-label mean estimates, and \citet{YLKJC13} developed an SVM approach with bag-level constraints. %
 Newer methods involve deep learning~\citep{KDFS15,DZCBV19,LWQTS19,NSJCRR22} and others leverage characteristics of the distribution of bags~\citep{SRR,ZWS22,chen2023learning,busafekete2023easy}. 
 The theoretical foundations of LLP were investigated by \citet{YCKJC14}, who defined the problem within the PAC framework and established bounds on the generalization error for the label proportion regression task. Recent work by \citet{Saket21}, \citet{Saket22} and \citet{brahmbhatt2023pac} addressed bag-classification using linear classifiers, providing algorithmic and hardness bounds. %
 Applications of LLP include privacy in online advertising~\citep{Obrien}, high energy physics~\citep{DNRS} and IVF predictions~\citep{hernandez2018}.

\noindent
{\bf Boosting.} The first boosting algorithm was given by \citet{Schapire} which was followed by a more efficient algorithm by \citet{Freund} %
and subsequently the famous AdaBoost algorithm~\citep{AdaBoost}. %
Further work \citep{xgboost, ent_lp_boost, brown_boost_Freund2001} resulted in the development of several boosting techniques, while \citet{AnyBoost} showed that several boosting algorithms (including AdaBoost~\citep{AdaBoost} and LogitBoost~\citep{LogitBoost}) implicitly perform gradient descent in the functional space and fall into the AnyBoost framework.

If we consider bags themselves as examples, one can directly apply existing boosting frameworks to obtain strong bag-level classifiers (see for e.g. \citep{two_view_llp_boost_Lai2023}). However, our goal is to obtain feature-vector level strong classifiers with high accuracy on bags. Previous works have adapted a subset of the above mentioned boosting approaches to LLP~\citep{ViolaPZ05,MILBoost,AdaBoostLLP} -- however they are empirically evaluated heuristics and not guaranteed to output strong classifiers.

\subsection{Problem Definition and Our Results} \label{sec:problemdefnandourresults}
Let $\bm{\mc{X}} \subseteq \R^d$ for some $d \in \mathbb{Z}^+$ be the space of feature-vectors, while a \emph{bag} $B$ is a finite subset of $\bm{\mc{X}}$. Let $\mc{Y} \subseteq \R$ be the space of feature-vector labels, and $\ol{\mc{Y}} \subseteq \R$ be the space of bag-level aggregate labels with some aggregation function ${\sf Agg}$ mapping finite $\mc{Y}$-valued tuples to $\ol{\mc{Y}}$. We say that a bag  $B = (\bx_1, \dots, \bx_q)$ with aggregate label $\sigma$ is \emph{satisfied} by a classifier $f : \bm{\mc{X}} \to \mc{Y}$ if ${\sf Agg}(f(\bx_1), \dots, f(\bx_q)) = \sigma$.  For convenience we will use bag to refer to a bag and its aggregate-label.

An $m$-sized \emph{training set} $\mc{B}$ is a collection $\{(B_j, \sigma_j) \in 2^{\bm{\mc X}} \times \ol{\mc{Y}}\}_{j=1}^m$ of $m$ bags and their aggregate-labels along with weights $w_j \geq 0$ for bag $B_j$ ($j=1,\dots, m$) such that $\sum_{j=1}^mw_j = 1$. 
The \textit{accuracy} of  a classifier on $\mc{B}$ is the weighted fraction of bags satisfied by it. We define a \emph{weak} classifier to be one with constant accuracy $\gamma > 0$, and a $\nu$-\emph{strong} classifier to have an accuracy $(1 - \nu)$. For ease of exposition we call the latter a strong classifier when $\nu$ can be taken to be an arbitrarily small positive constant.

For this study, the underlying feature-vector level task is binary classification, so $\mc{Y} = \{0,1\}$. 
For multiple instance learning (MIL) the aggregation function is ${\sf OR}$ i.e., the boolean disjunction and therefore $\ol{\mc{Y}} = \{0,1\}$. On the other hand, in learning from label proportions (LLP) we take the aggregation function to be ${\sf SUM}$ i.e., the real sum of labels, and therefore $\ol{\mc Y} = \Z^{\geq 0}$. Note that for LLP we could have equivalently taken average as the aggregation (since the size of any bag is known), however for convenience we use ${\sf SUM}$.

We also define the ${\sf Trv}_{\sf LLP}(\mc{B})$ for a collection of LLP bags, to denote the trivial accuracy threshold on $\mc{B}$. Specifically, it is the minimum weighted accuracy given by the best among the random classifier and the two constant valued classifiers (all $0$s and all $1$s classifiers), over all possible weight assignments to the bags $\mc{B}$. For a collection of MIL bags $\mc{B}$, ${\sf Trv}_{\sf MIL}(\mc{B})$ is defined analogously. 

We shall also use the \emph{halfspace} classifier whose value at point $\bx \in \R^d$ is given by ${\sf pos}\left(\langle \br, \bx \rangle + c\right)$ for some $\br \in \R^d$, $c \in \R$ where $\pos(a) = 1$ if $a > 0$ and $0$ otherwise. We say that the halfspace passes through the origin i.e., is \textit{homogeneous} if $c = 0$. Next we state this paper's results.

\subsubsection{Our Results} 
We begin with the impossibility results for boosting in the LLP (Theorem \ref{thm:LLP-imposssibility}) and MIL (Theorem \ref{thm:MIL-imposssibility}) settings. These theorems coupled with the definition of the boosting meta algorithm (Section \ref{sec:preliminaries_boosting}) imply our impossibility results.
\begin{theorem}[Impossibility of boosting in LLP]\label{thm:LLP-imposssibility}
    Let $\alpha \in [1/2, 1)$ be any constant. Then, for any arbitrarily small constant $\eps > 0$ there exists $d, m \in \Z^+$ and a collection of bags $\mc{B} = \{B_j \subseteq \R^d\}_{j=1}^m$ where $|B_j| = 2$ and the aggregate label (i.e. sum of labels in LLP setting) of $B_j$ is $1$ ($j = 1,\dots, m$) and the following properties are satisfied: 
    
\noindent
    \tn{(Existence of weak halfspace classifiers):} For any assignment of weights $w_j$ to $B_j$ ($j = 1,\dots, m$) such that $\sum_{j=1}^m w_j = 1$, for the weighted collection of bags there is a halfspace classifier with accuracy $\alpha$.
 
\noindent   
    \tn{(No Strong Classifier):} For the unweighted set of bags $\{B_j \subseteq \R^d\}_{j=1}^m$ there is no classifier $f : \cup_{j=1}^m B_j \to \{0,1\}$ having accuracy greater than $\alpha + \eps$.
\end{theorem}
The above theorem, proved in Section \ref{sec:impossibility_llp}, is optimal from multiple perspectives: firstly the bags are of size at most $2$ whereas when bags are all of size $1$ (i.e., supervised learning) boosting is indeed possible, showing that as soon as we transition from the fully supervised to the LLP setting in terms of bag size, boosting becomes impossible. Secondly, the result shows that even if weak learners of \emph{any} constant accuracy in $[1/2, 1)$ exist, there is no classifier with even a slightly greater accuracy, thus ruling out any non-trivial advantage of boosting, let alone obtaining a strong classifier. In Appendix \ref{app:trivial_performance} we give a simple argument showing that ${\sf Trv}_{\sf LLP}(\mc{B}) = 1/2$ for the bags $\mc{B}$ constructed in the above theorem.
We now state our result (proved in Section \ref{sec:impossibility_mil}) on the impossibility of boosting in the MIL setting.
\begin{theorem}[Impossibility of boosting in MIL]\label{thm:MIL-imposssibility}
    For any arbitrarily small constant $\eps > 0$ there exist $m \in \Z^+$ and a collection of bags $\mc{B} = \{B_j \subseteq \R^d\}_{j=1}^m$ along with the aggregate labels $\sigma_j$ for $B_j$ where $|B_j| = 2$ ($j = 1,\dots, m$) and the following properties are satisfied: 

\noindent
    \tn{(Existence of weak halfspace classifiers):} For any assignment of weights $w_j$ to $B_j$ ($j = 1,\dots, m$) such that $\sum_{j=1}^m w_j = 1$, for the weighted collection of bags there is a halfspace classifier with accuracy $2/3 - \eps$.

\noindent
    \tn{(No Strong Classifier):} For the unweighted set of bags $\{B_j \subseteq \R^d\}_{j=1}^m$ there is no classifier $f : \cup_{j=1}^m B_j \to \{0,1\}$ having accuracy greater than $3/4$.
\end{theorem}
The above theorem shows that in the MIL setting, weak classifiers with any accuracy $< 2/3$ cannot be boosted to a strong classifier with accuracy $> 3/4$. As shown in 
 Appendix \ref{app:trivial_performance}, ${\sf Trv}_{\sf MIL}(\mc{B}) = 1/2$ for the bags $\mc{B}$ of the above theorem, and therefore our result applies to w non-trivial weak classifier accuracy in $(1/2, 2/3)$.

Next we state our results (proved in Section \ref{sec:weak_to_strong_alg}) in the LLP setting for obtaining a strong classifier on a collection of bags using a weak classifier on a derived collection of larger bags. In this case we consider unweighted collection of bags, since a weighted collection of $m$ bags can easily be converted into an unweighted collection of $Tm$ bags while preserving the accuracy of any classifier up to an additive error of $O(1/T)$ (see Appendix \ref{app:wtdtounwtd}). 
To state our result we assume that there is an oracle $\mc{O}_{q,\alpha}(\ol{\mc{B}})$ which given weighted collection of bags $\ol{\mc{B}}$ along with their aggregate labels, where each bag has size at most $q$, outputs a classifier $f$ with weighted accuracy $\alpha$ on $\ol{\mc{B}}$. 

\begin{theorem}[Weak to Strong LLP Learning]\label{thm:weaktostrong}
    For parameters $\alpha, \eps > 0$ there exists $t = O(1/(\eps\alpha^2))$, and algorithms $\mc{A}_1$ and $\mc{A}_2$ s.t. given an unweighted collection of $m$ bags $\mc{B}$, where $k = \max_{(B,\sigma)\in \mc{B}}\left|B\right|$ and $n := \left|\cup_{(B, \sigma)\in \mc{B}}B\right|$, and assuming that $\mc{O}_{kt,\alpha}$ exists,
    \begin{itemize}[leftmargin=1em]
   \item $\mc{A}_1$ creates a weighted collection  $\ol{\mc{B}}_1$ of at most $m^{t+1}$ bags each of size at most $kt$ such that $\mc{O}_{kt,\alpha}(\ol{\mc{B}}_1)$ outputs a classifier which has accuracy $(1-\eps)$ on $\mc{B}$.
    \item for any $\delta > 0$, $\mc{A}_2$ creates a random collection $\ol{\mc{B}}_2$ of $s = O\left(\frac{1}{\alpha}\left(n + \log\left(\frac{1}{\delta}\right)\right)\right)$ each of size at most $kt$ such that $\mc{O}_{kt, \alpha}(\ol{\mc{B}}_2)$ has accuracy $(1-\eps)$ on $\mc{B}$ with probability at least $(1-\delta)$. If $\mc{O}_{kt, \alpha}$ is guaranteed to output a classifier of VC dimension $r$ then $s = O\left(\frac{r}{\alpha}\log\left(\frac{n}{r}\right) + \log\left(\frac{1}{\delta}\right)\right)$ suffices.
    \end{itemize}
\end{theorem}

Theorem \ref{thm:weaktostrong} presents algorithms that, when applied to collections of bags in the LLP setting, yields high-accuracy classifiers by employing weak classifiers trained on a reasonably sized collections of large bags. This can in particular be achieved by an efficient randomized algorithm $\mc{A}_2$. We also conduct experiments (see Section \ref{sec:experiments}) -- on both real and synthetic datasets -- to demonstrate the effectiveness of $\mc{A}_2$. We use it to construct a limited collection of large bags from a given collection of small bags and experimentally show that a weak classifier on the large bags yields one with significantly higher accuracy on the constituent small bags.

\subsection{Overview of Techniques}\label{sec:overview}
{\bf Impossibility of Boosting in LLP} (Theorem \ref{thm:LLP-imposssibility}). Our construction follows from  the well-known \textit{semi-definite programming} (SDP) integrality gap of \citet{Feige2002} for the Max-Cut problem. In this, for some arbitrarily small $\eps > 0$, with $d$ depending on $\eps$, the vertices of the graph are given by points on the $(d-1)$-dimensional unit sphere $\mathbb{S}^{d-1}$. For any constant $\alpha \in [1/2, 1)$, each edge is between points that are at an angle of at least $\alpha\pi$. Using techniques related to spherical isoperimetry and concentration of measure in high dimensions, the authors prove that there is no cut in the graph separating more than $(\alpha + \eps)$-fraction of the edges. By creating a $2$-sized bag corresponding to each edge with latter's two end-points being the bag's two feature-vectors, we create a collection of bags, and for each one we assign an aggregate label $1$ i.e., any bag is satisfied if exactly one of its feature-vectors is labeled $1$ or equivalently the corresponding edge is separated. The cut upper bound of $(\alpha + \eps)$ thus directly gives us the upper bound on the best possible accuracy of any classifier. On the other hand, since the angle between the feature-vectors of any edge is at least $\alpha\pi$, a random halfspace passing through the origin -- given by ${\sf pos}\left(\br^{\sf T}\bx\right)$ for a random unit vector $\br$ -- has expected accuracy $\alpha$ for any weight assignment to the bags, and therefore there is some halfspace achieving accuracy $\alpha$.

\noindent
{\bf Impossibility of Boosting in MIL} (Theorem \ref{thm:MIL-imposssibility}). Since the aggregation function is ${\sf OR}$ the Max-Cut construction of \cite{Feige2002} is not applicable. Instead we hand-craft the set of bags as follows. The set of feature-vectors is all points on the unit circle $\mathbb{S}^1$  and for some $\alpha \in (1/2, 1)$, we create a bag with two points if the angle between them is exactly $\alpha \pi$ and give an aggregate label $1$ to all such two sized bags (let us call them $1$-bags). We also construct $2$-sized bags with aggregate label $0$ when the angle between two points is exactly $(1 - \alpha)\pi$ (called as $0$-bags). %
If we consider any reweighted collection of these bags then %
a simple threshold based case-analysis yields weak classifier of accuracy $2/3 - (1-\alpha)/2$. To rule out any strong classifier, we consider a labeling where $z$-fraction of the points in $\mathbb{S}^1$ are labeled as 1. We show that the maximum accuracy possible is 3/4 which is achieved at $z = 1/2$. We choose $\alpha = 1 - \eps$  while losing an additional error of $\eps/2$ in the weak-classifier accuracy due to discretization to obtain the desired bounds. 

\noindent
{\bf Weak to Strong LLP Learning} (Theorem \ref{thm:weaktostrong}). The main idea is, given a target collection of bags $\mc{B}$, to construct all possible bags which are unions of up to $t$ bags from $\mc{B}$. Note that the aggregate label for the union is simply the sum of the aggregate labels of the constituent bags, and the error of a classifier w.r.t. the aggregate label on the union of bags is the sum of errors on the constituent bags. Let $f$ be a classifier with accuracy $\gamma > 0$ on these larger bags, and assume for a contradiction that $f$ has accuracy less than $(1- \eps)$ on $\mc{B}$, for some $\eps > 0$. Call those bags in $\mc{B}$ on which $f$ has a non-zero error  $\in \mathbb{Z}\setminus\{0\}$ w.r.t. the aggregate label, as the \emph{error} bags. Now, if $t$ is large enough then a random set of $t$ bags from $\mc{B}$ has, with high probability $\approx \eps t$ error bags. Using a sampling argument we show that the error on the union of $t$ random bags from $\mc{B}$ is distributed like a random Bernoulli combination of the errors on $\approx 2\eps t$ bags. We then apply the Littlewood-Offord-Erd\H{o}s anti-concentration lemma to obtain that with probability at least $(1 - O(1/(\sqrt{\eps t}))$, the the union of the bags has non-zero error induced by $f$. By choosing $t$ large enough we obtain a contradiction with the accuracy of $\alpha$ on the large bags. Standard sampling techniques can be applied to obtain a more efficient procedure with high probability guarantees.

\section{Preliminaries} \label{sec:preliminaries}

\begin{lemma}[Chernoff Bounds] Let $X = \sum_{i=1}^{n} X_i$, where $X_i$ = 1 with probability $p_i$ and $X_i = 0$ with probability $1 - p_i$, and all $X_i$ are independent. Let $\mu = \mathbb{E}(X) = \sum_{i=1}^{n} p_i$. Then (i) Lower Tail: $\Pr[X \leq (1 - \eta)\mu] \leq e^{-\eta^2\mu/2} \ \forall \ 0 < \eta < 1$, and (ii) Upper Tail: $\Pr[X \leq (1 + \eta)\mu] \leq e^{-\eta^2\mu/(2 + \eta)} \ \forall \ 0 \leq \eta$.
\label{lemma:chernoff_bounds}
\end{lemma}

\begin{lemma}[Littlewood-Offord-Erd\H{o}s Lemma~\cite{littlewood_offord_Erds1945OnAL}] Let $X_1, X_2, \dots , X_n$ be \emph{i.i.d} $\{0, 1\}$-Bernoulli random variables with $\Pr[1] = 1/2$, and let $a_1, a_2, . . . , a_n \in \mathbb{R}$ s.t. $|a_i| \geq 1, \ \forall i \in[n]$. Then, there exists an absolute constant $C > 0$ such that
\begin{equation*}
	\underset{X_1, \dots, X_n}{\Pr} \left[ \left| \sum_{i \in [n]} a_i X_i + \theta \right| \leq 1 \right] \leq \frac{C}{\sqrt{n}}
\end{equation*}
for any constant $\theta$.
\label{lemma:littlewood_offord}
\end{lemma}

\begin{theorem}[ Theorem 3.7 from \cite{anthony_bartlett_1999neural}] For a $\{0,1\}$-valued class $\mc{H}$ of functions with VC-dimension $\tn{VC-dim}(\mc{H}) = v$, let $\Pi_{\mc{H}}(n)$ denote the maximum number of possible $\{0,1\}$-labelings to any set of $n$ points from the domain of $\mc{H}$. If $n \leq v$, $\Pi_H(n) \leq 2^n$ and for $n > v$, $(\frac{en}{v})^v$. Refer to Section 3.3 of \cite{anthony_bartlett_1999neural} for more details on VC Dimension.
\label{theorem:vcdim_growth_function}
\end{theorem}

\subsection{Boosting meta algorithm for aggregate label setting}
\label{sec:preliminaries_boosting}
Given a collection of bags and aggregate labels, a prototypical boosting algorithm (given in Figure \ref{algo:boosting}) in the aggregate label setting, involves repeating certain steps over some number of rounds: in each round the training data is reweighed, for which a weak classifier is computed. The final output is some function over the ensemble of computed weak classifiers.
\begin{figure}[!htb]
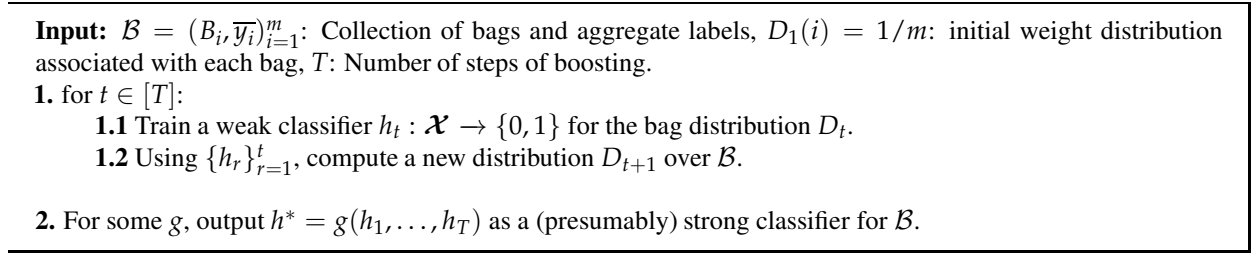

\begin{mdframed}
\small
\textbf{Input:} $\mc{B} = (B_i, \bar{y_i})_{i=1}^m$: Collection of bags and aggregate labels, $D_1(i) = 1/m$: initial weight distribution associated with each bag, $T$: Number of steps of boosting.
\\ \textbf{1.} for  $t \in [T]$: \\
 \hspace*{2em} \textbf{1.1} Train a weak classifier $h_t: \bm{\mc{X}} \xrightarrow[]{} \{0, 1\}$ for the bag distribution $D_t$. \\
\hspace*{2em} \textbf{1.2} Using $\{h_r\}_{r=1}^t$, compute a new distribution $D_{t+1}$ over $\mc{B}$. \\ %
\\ \textbf{2.} For some $g$, output $h^* = g(h_1, \dots, h_T)$ as a (presumably) strong classifier for $\mc{B}$.  
\end{mdframed}
\caption{Boosting for aggregate label setting}\label{algo:boosting}
\end{figure}

\section{Impossibility of Boosting in LLP} \label{sec:impossibility_llp}

The {\sf Max-Cut} problem is: given an undirected graph $G(V, E)$ find a cut given by the assignment $g : V \to \{0,1\}$ which separates the maximum number of edges in $E$ i.e., maximizes $\left|\{e = \{u, v\} \in E| g(u) \neq g(v)\}\right|$. We shall use the following construction of graph $G_{\tn{FS}}(V_{\tn{FS}}, E_{\tn{FS}})$ given in Sec. 3.1 of \cite{Feige2002}: 

{\bf Construction.} Let $\alpha\pi = \theta \in [\pi/2, \pi)$ and $\eps > 0$ be an arbitrarily small parameter such that $\theta + \eps\pi < \pi$. Let $d = O(1/\eps\log(1/\eps)$ and $\gamma = \eps^2/(2d)$. Divide the $(d-1)$-dimensional unit sphere $\mathbb{S}^{d-1}$ into $\left(\frac{O(1)}{\gamma}\right)^d$ equal sized cells of diameter at most $\gamma$ each (this is shown to be possible in Lemma 21 of \cite{Feige2002}). From each cell pick an arbitrary point $\bv$ and add it to $V_{\tn{FS}}$. Add an edge $\{\bu, \bv\}$ to $E_{\tn{FS}}$ for each pair of points $\bu, \bv \in V_{\tn{FS}}$ whose angle is between $\theta$ and $\theta + \eps$.

Section 3.1 of \cite{Feige2002} shows\footnote{While \cite{Feige2002} state the proof of \eqref{eqn:FSoptd} for a specific value of $\theta$, the proof applies to all values of $\theta \in [\pi/2,\pi)$.} that
\begin{equation}
    \Pr_{\{\bu, \bv\} \in E_{\tn{FS}}}[g(\bu) \neq g(\bv)] \leq \theta/\pi + O(\eps^2)  = \alpha + O(\eps^2) \label{eqn:FSoptd}
\end{equation}
for any $g : V_{\tn{FS}} \to \{0,1\}$. 

\subsection{Proof of Theorem \ref{thm:LLP-imposssibility}}
Let $G_{\tn{FS}}(V_{\tn{FS}}, E_{\tn{FS}})$ be the graph constructed above using $\theta = \alpha\pi \in [\pi/2, \pi)$ and let $\eps$ taken to be the same as that in the statement of Theorem \ref{thm:LLP-imposssibility}. Taking $V_{\tn{FS}}$ to be the underlying set feature-vectors, let the set of bags $\mc{B}$ be $E_{\tn{FS}}$ i.e., each edge $\{\by, \bv\}$ is a bag. All aggregate labels are $1$, so that any bag is satisfied by $g : V_{\tn{FS}} \to \{0,1\}$ iff the corresponding edge is separated by $g$. 

Now, for any bag $\{\bu, \bv\}$ in $\mc{B}$, from the construction of $G_{\tn{FS}}(V_{\tn{FS}}, E_{\tn{FS}})$, the angle between $\bu$ and $\bv$ is at least $\theta$. Thus, a random homogeneous halfspace (given by ${\sf pos}\left(\br^{\sf T}\bx\right)$ for $\br$ chosen uniformly at random from $\mathbb{S}^{d-1}$) satisfies the bag with probability at least $\theta/\pi = \alpha$. 

Thus for any assignment of weights $w_B$ for bags $B \in \mc{B}$, the expected weight of bags satisfied by a random  homogeneous halfspace is $\sum_{B}w_B\Pr_{\br \leftarrow \mathbb{S}^{d-1}}\left[B\tn{ is satisfied by } {\sf pos}\left(\br^{\sf T}\bx\right)\right] = \alpha\sum_Bw_B$ by linearity of expectation. Therefore, there is one classifier with weighted accuracy $\alpha$.

The upper bound on the accuracy of \emph{any} classifier on $\mc{B}$ follows directly from \eqref{eqn:FSoptd} and small enough $\eps > 0$.

\section{Impossibility of Boosting in MIL} \label{sec:impossibility_mil}

Along similar lines as the previous section, we provide a geometric construction of MIL on $2$-sized bags. We begin with a continuous set of points which we analyze and subsequently discretize while preserving its key properties. We fix a parameter $\alpha \in (1/2, 1)$.

\noindent
{\bf Construction.} Let $\bm{\mc{X}}_c$ be set of all points on the unit circle $\mathbb{S}^{1}$. For any two points that subtend an angle of exactly $\alpha \pi$ we create a $2$-sized bag with aggregate label $1$ (we call it a $1$-bag) containing those points. Similarly, bags with aggregate label $0$ (which we call $0$-bags) are formed by pairs of points at an angle of $(1-\alpha)\pi$. 
By mapping a $1$-bag to the mid-point of the smaller arc subtended by the two points in the bag (end-points), and noting that all the $1$-bags have unique mid-points, we obtain that the measure of the set of $1$-bags is same as that of $\mathbb{S}^1$. Similarly, this holds true for the set of $0$-bags. 
In particular, the set of $0$-bags and the set of $1$-bags are of equal measure. Let $\mc{B}_c$ be this infinite (continuous) collection of $1$-bags and $0$-bags. 

\noindent
{\bf Existence of Weak Classifier.}
Observe that the constant $0$ classifier given by ${\sf pos}(-1)$ will satisfy all $0$-bags and none of the $1$-bags.

Now, consider a random homogeneous halfspace given by ${\sf pos}(\br^{\sf T}\bx)$ for $\br$ uniformly sampled from $\mathbb{S}^1$. The two points of a $0$-bag will not be separated w.p. $\alpha$ and with a further $1/2$ both will be assigned $0$, implying that any $0$-bag will be satisfied with probability $\alpha/2$. On the other hand, both the points of a $1$-bag will be assigned $0$ w.p. $(1 - \alpha)/2$ implying that it will be satisfied w.p. $(1 + \alpha)/2$. 

Let there be any probability measure on $\mc{B}_c$ s.t. the measure of the $0$-bags is $p$ and that of the $1$-bags is $(1-p)$. If $p \geq 2/3$ then the constant $0$ classifier satisfies all the $0$-bags yielding an accuracy of $p \geq 2/3$. If not, then the random homogeneous halfspace satisfies in expectation
\begin{eqnarray}
    p\alpha/2 + (1-p)(1 + \alpha)/2 & = & (1 + \alpha)/2 - p/2 \nonumber \\ & \geq & 1/2 + \alpha/2 - 1/3 \nonumber \\ & = & 2/3 - (1 -\alpha)/2 \label{eqn:weak-multi}
\end{eqnarray}
Therefore, there is always a weak classifier, for any reweighing of the bags, of accuracy $2/3 - (1-\alpha)/2$.

\noindent
{\bf No Strong Classifier.}
Consider any $\{0,1\}$-labeling of $\mathbb{S}^1$.
Let $z \in [0, 1]$ represent the fraction of points on $G_c$ labeled as 1, with the remaining fraction $1-z$ labeled as 0. Sampling a $0$-bag u.a.r. and randomly choosing one of its points yields the uniform distribution over $\mathbb{S}^1$. Thus, the probability that a random $0$-bag  is satisfied is $\leq 1 - z$. 
The two points of all the $1$-bags cover $\mathbb{S}^1$ twice, so the probability that in a random $1$-bag at least one of its points is labeled $1$ is at most $\min\{2z, 1\}$.

Therefore, the probability that a random bag from $\mc{B}_c$ is satisfied by the labeling is at most
\begin{equation}
    \frac{1-z + \min\{2z,1\}}{2} = \begin{cases} 1 - z/2 & \tn{ if } z \geq 1/2 \\
                                                 1/2 + z/2 & \tn{ otherwise}
                                    \end{cases}
\end{equation}
which attains a maximum of $3/4$ at $z = 1/2$. Thus, no classifier can have accuracy $ > 3/4$ on $\mc{B}_c$

\noindent
{\bf Discretization.} Let $T$ be a large positive integer, and divide $\mathbb{S}^1$ into $2T$ continuous, non-overlapping arcs $\{A_i\}_{i=1}^{2T}$ of length $\delta \pi$ each, where $\delta = 1/T$. We choose $T$ large enough so that $2\delta < \min\{(2\alpha -1), (1 - \alpha)\}$, ensuring that: \\
(i) there is no segment that contains both endpoints of any bag in $\mc{B}_c$, and \\
(ii) for any pair of segments $A_i$ and $A_j$, if there is a $0$-bag in $\mc{B}_c$ with one point in $A_i$ and another in $A_j$, then there is no such $1$-bag, and similarly if there is a $1$-bag in $\mc{B}_c$ with one point in $A_i$ and another in $A_j$, then there is no such $0$-bag.

Using property (ii) above, let us construct a discrete set of bags $\mc{B}_d$ as follows. If a pair of segments $A_i$ and $A_j$ are such that there is a $0$-bag in $\mc{B}_c$ with one point in $A_i$ and another in $A_j$, then add $\{A_i, A_j\}$ as $0$-bag with weight as the measure of all the bags in $\mc{B}_c$ (which are necessarily $0$-bags) with one point in $A_i$ and another in $A_j$. Analogously, add pairs of segments as $1$-bags.
Note that from property (i), all bags in $\mc{B}_d$ have size $2$.

Let us first consider any $\{0,1\}$-labeling to $\{A_i\}_{i=1}^{2T}$. This directly corresponds to a $\{0,1\}$-labeling to $\mathbb{S}^1$ by assigning a point the label of the segment containing it. Further, from its construction, the weight of the bags $\mc{B}_c$ satisfied by the labeling to the segments equals the measure of the bags in $\mc{B}_c$ satisfied by the corresponding labeling to $\mathbb{S}^1$ which, as shown above, is at most $3/4$.

In particular, the above argument also shows that the measure of bags in $\mc{B}_d$ satisfied by the constant $0$ labeling to $\{A_i\}_{i=1}^{2T}$ is the same as that in $\mc{B}_c$ satisfied by the constant $0$ labeling to  $\mathbb{S}^1$.

Lastly, we translate the labeling by a homogeneous halfspace on  $\mathbb{S}^1$ to a labeling for $\{A_i\}_{i=1}^{2T}$ by assigning each $A_i$ the label of its mid-point. Consider the \emph{error} set of points in $\mathbb{S}^1$ whose label given by the homogeneous halfspace differs from the label of the segment containing it. For any homogeneous halfspace, the error set is entirely contained within the two diametrically opposite segments intersected by the halfspace. Similarly, the \emph{error} bags in $\mc{B}_c$ are those  whose aggregate label given by the homogeneous halfspace differs from the aggregate label of the corresponding bag in $\mc{B}_d$.

The \emph{error} bags in $\mc{B}_c$ are a subset of those which have at least one end-point in the the error set of points. Given any bag in $\mc{B}_c$ the probability over a random homogeneous halfspace that it is an error bag is at most the probability that one of its endpoints is in a segment intersected by the halfspace. By symmetry, a segment is intersected with probability $1/T$. So the probability that any bag in $\mc{B}_c$ is an error bag is at most $2/T = 2\delta$.

Thus, from \eqref{eqn:weak-multi} we obtain that for any weighing of the bags in $\mc{B}_d$, there is a classifier of accuracy $2/3 - (1-\alpha)/2 - 2\delta$.

\subsubsection{Completing the proof of Theorem \ref{thm:MIL-imposssibility}.} For this, we can take $\eps$ to be small enough, say $\eps \in (0, 0.1)$ and set $\alpha = 1 - \eps$ along with $T = \lceil 4/\eps\rceil$ so that $\delta \leq \eps/4$ and $2\delta < \min\{(2\alpha -1), (1 - \alpha)\}$ and $2/3 - (1-\alpha)/2 - 2\delta \geq 2/3 - \eps$.

\section{Weak to Strong Classification in LLP} \label{sec:weak_to_strong_alg}

Given $\alpha, \eps > 0$ we set $t$ to be $\frac{32}{\eps}\left(\frac{C_0}{\alpha}\right)^2$ where $C_0 > 0$ is an absolute constant to be decided. 
We begin by defining in Fig. \ref{algo:DistnDbar} a distribution $\ol{D}$ over bags $(\ol{B}, \ol{\sigma})$ where $\ol{B}$ is the union of at most $t$ bags from $\mc{B}$ and $\ol{\sigma}$ is the sum of their aggregate labels. 

\begin{figure}[!htb]
\begin{mdframed}
\small
\textbf{Input:} : Bags $\mc{B}$, $t$.\\
\textbf{Steps:}
\begin{enumerate}
    \item Independently for $i = 1, \dots, t$, let $\mc{P}_i = (B_i, \sigma_i)$ where $(B_i, \sigma_i)$ is sampled u.a.r. from $\mc{B}$.
    \item Independently for $i = 1, \dots, t$: set $\mc{Q}_i = \mc{P}_i$ w.p. $1/2$ and set $\mc{Q}_i = \star$ w.p. $1/2$.
    \item Output $(\ol{B}, \ol{\sigma})$ where
    \begin{equation}
    \displaystyle \ol{B} = \underset{\{i\,\mid\,\mc{Q}_i = (B_i, \sigma_i) \neq \star\}}{\bigcup} B_i, \tn{\ \ and\ \ } \ol{\sigma} = \underset{\{i\,\mid,\mc{Q}_i = (B_i, \sigma_i) \neq \star\}}{\sum}\sigma_i
    \end{equation}\label{eqn:olB}
\end{enumerate}
\end{mdframed}
\caption{Distribution $\ol{D}$.}\label{algo:DistnDbar}
\end{figure}

To aid our subsequent analysis we shall use the following straightforward lemma.
\begin{lemma}\label{lem:chernofappl}
   For any subset $\mc{S} \subseteq \mc{B}$ s.t. $|\mc{S}| \geq \kappa |\mc{B}|$, in Step 1. of Fig. \ref{algo:DistnDbar}, $\Pr\left[|\{i\,\mid\,(B_i, \sigma_i) \in \mc{S})\}| < \kappa t/2\right] \leq \tn{exp}(-\kappa t/8)$.
\end{lemma}
\begin{proof}
    Since each $(B_i, \sigma_i)$ independently belongs to $\mc{S}$ w.p. $\kappa$,  $\Pr[(B_i, \sigma_i) \in \mc{S}] \geq \kappa$ and therefore $\mu := \E\left[|\{i\,\mid\,(B_i, \sigma_i) \in \mc{S})\}| \right] \geq \kappa t$. Thus, $\Pr\left[|\{i\,\mid\,(B_i, \sigma_i) \in \mc{S})\}| < \kappa t/2\right] \leq \Pr\left[|\{i\,\mid\,(B_i, \sigma_i) \in \mc{S})\}| < \mu/2\right] \leq \tn{exp}(-\mu/8) \leq \tn{exp}(-\kappa t/8)$, where we use the Chernoff Tail Bound (Lemma \ref{lemma:chernoff_bounds}) using $\eta = 1/2$ and the lower bound of $\kappa t$ for $\mu$.
\end{proof}

\subsection{Analysis for a fixed classifier $h$}
We prove the following lemma.
\begin{lemma}\label{lem:errorampl}
    Let $h: \bm{\mc{X}} \to \{0,1\}$ be a classifier such that $h$ has accuracy $< (1- \zeta)$ on $\mc{B}$. Then, 
$$\Pr_{\ol{B}, \ol{\sigma}) \leftarrow \ol{D}}\left[\sum_{\bx in \ol{B}}h(\bx) = \ol{\sigma} \right] \leq C_0/\sqrt{\zeta t} + \tn{exp}(-\zeta t/8)$$
for some absolute constant $C_0 > 0$.
\end{lemma}
\begin{proof}
    Let $\mc{B}_{\tn{err}}$ be the \emph{error} bags $(B, \sigma) \in \mc{B}$ on which $\sum_{\bx \in B} h(\bx) \neq \sigma$, so that $|\mc{B}_{\tn{err}}| \geq \zeta |\mc{B}|$. For convenience, we shall abuse the notation $h(B)$ to denote $\sum_{\bx \in B}h(\bx)$, and therefore, for an error bag $B$, $\left|h(B) - \sigma\right| \geq 1$. Depending on the choices in Step 1. of Fig. \ref{algo:DistnDbar}, define the set $I := \{i\,\mid\,(B_i, \sigma_i) \in \mc{B}_{\tn{err}})\}$ and let $E_0$ be the event that the following occurs: $\left\{ |I| \geq \zeta t/2\right\}$. Further, let $E_1$ be the event that the following occurs:
\begin{equation}
        h(\ol{B}) = \bar{\sigma} \Leftrightarrow \underset{\{i\,\mid,\mc{Q}_i = (B_i, \sigma_i) \neq \star\}}{\sum}\left(h(B_i) - \sigma_i\right) = 0 \label{eq:E1event}
\end{equation}
where $(\ol{B}, \ol{\sigma})$ is the output in Step 3. Now, 
\begin{eqnarray}
\Pr[E_1] = & \Pr[E_1 | E_0]\Pr[E_0] + \Pr[E_1 | \neg E_0]\Pr[\neg E_0] \nonumber \\ \leq & \Pr[E_1 | E_0] + \Pr[\neg E_0] \nonumber
\end{eqnarray}
Since $|\mc{B}_{\tn{err}}| \geq \zeta |\mc{B}|$, Lemma \ref{lem:chernofappl} yields that $\Pr[\neg E_0] \leq \tn{exp}(-\zeta t/8)$. On the other hand, fix the set $I$ and bags $\{(B_i, \sigma_i)\}_{i\in I}$ and let $a_i := h(B_i) - \sigma_i$ ($i = 1, \dots t$). Defining $\{X_i\,\mid\, i \in I\}$ to be i.i.d $\{0,1\}$-valued Bernoulli random variables which are $1$ w.p. $1/2$, we obtain that $\Pr[E_1] = \Pr[\sum_{i \in I}a_iX_i = 0] \leq C/\sqrt{|I|}$ by applying Lemma \ref{lemma:littlewood_offord}. Therefore, $\Pr[E_1 | E_0] \leq C/\sqrt{(\zeta /2)t}$  and using the above bounds, $\Pr[E_1] \leq  C/\sqrt{(\zeta/2) t} + \tn{exp}(-\zeta t/8)$. %
\end{proof}

\subsection{Deterministic algorithm $\mc{A}_1$}\label{sec:A_1}
\begin{figure}[!htb]
\begin{mdframed}
\small
\textbf{Input:} : Bags $\mc{B}$, $k = \max_{(B,\sigma) \in \mc{B}} |B|$, $\alpha > 0$, $t$, oracle $\mc{O}_{kt, \alpha}$.\\
\textbf{Steps:}
\begin{enumerate}
    \item Let ${\sf supp}(\ol{D})$ be the support of $\ol{\mc{D}}$ (Fig. \ref{algo:DistnDbar}), and for each $(\ol{B}, \ol{\sigma}) \in {\sf supp}(\ol{D})$ let its weight $w_{(\ol{B}, \ol{\sigma})}$ be its probability under $\ol{D}$. Let $\ol{\mc{B}}$ be ${\sf supp}(\ol{D})$ with weights  $w_{(\ol{B}, \ol{\sigma})}$. 
    \item Output the classifier $h^*$ given by  $\mc{O}_{kt,\alpha}(\ol{\mc{B}})$.
\end{enumerate}
\end{mdframed}
\caption{Algorithm $\mc{A}_1$.}\label{algo:A1}
\end{figure}
Figure \ref{algo:A1} describes algorithm $\mc{A}_1$ using\footnote{We include in Appendix \ref{sec:suppD} an explanation on computing the probabilities under $\ol{D}$.} the distribution $\ol{D}$ defined in Figure \ref{algo:DistnDbar}. Suppose for a contradiction that the output $h^*$ of $\mc{A}_1$ has accuracy $< (1- \eps)$ on $\mc{B}$. Then, from Lemma \ref{lem:errorampl} we obtain that the probability that $(\ol{B}, \ol{\sigma})$ sampled from $\mc{D}$ is satisfied by $h^*$ is at most $C_0/\sqrt{\eps t} + \tn{exp}(-\eps t/8)$ which -- upon plugging in the value of $t$ -- is at most $\alpha/2$ which contradicts the accuracy of $h^*$ on $\ol{\mc{B}}$.

We next describe a more efficient, albeit randomized, variant of the algorithm.

\subsection{Randomized algorithm $\mc{A}_2$} \label{sec:A2}
\begin{figure}[!htb]
\begin{mdframed}
\small
\textbf{Input:} : Bags $\mc{B}$, $k = \max_{(B,\sigma) \in \mc{B}} |B|$, $\alpha > 0$, $t$, oracle $\mc{O}_{kt, \alpha}$, $s \in \mathbb{Z}^+$.\\
\textbf{Steps:}
\begin{enumerate}
    \item Let $\hat{\mc{B}} = \{(\hat{B}_j,\hat{\sigma}_j)\}_{j=1}^s$ be $s$ i.i.d. samples from $\ol{D}$ (Fig. \ref{algo:DistnDbar}). 
    \item Output the classifier $\tilde{h}$ given by  $\mc{O}_{kt,\alpha}(\hat{\mc{B}})$.
\end{enumerate}
\end{mdframed}
\caption{Algorithm $\mc{A}_2$.}\label{algo:A2}
\end{figure}
Figure \ref{algo:A2} provides the algorithm $\mc{A}_2$.
Fix any $h$ that has accuracy $< (1- \eps)$ on $\mc{B}$. Then, by Lemma \ref{lem:errorampl}, and our setting of $t$ we obtain that $\Pr_{(\hat{B}, \hat{\sigma})\leftarrow \ol{D}}[(\hat{B}, \hat{\sigma})\tn{ satisfied by } h] \leq \alpha/2$. Therefore, in Step 1 of $\mc{A}_2$ it is easy to see by monotonicity that 
\begin{eqnarray}
  \Pr\left[\left|\{j \in [s]\,\mid\, (\hat{B}_j,\hat{\sigma}_j) \tn{ satisfied by h}\}\right| \geq \alpha s\right] \leq \P\left[\sum_{\ell = 1}^s X_\ell \geq \alpha s\right] \label{eqn:randomAlg-1}
\end{eqnarray}
where each $X_\ell$ ($\ell = 1, \dots, s$) is an independent $\{0,1\}$-valued Bernoulli random variable taking value $1$ with probability $\alpha /2$. Therefore, using Chernoff Upper Tail bound from Lemma \ref{lemma:chernoff_bounds} we can upper bound the LHS of \eqref{eqn:randomAlg-1} by $\tn{exp}(-\alpha s/6)$ which is the upper bound on the probability that $h$ has accuracy $\geq \alpha$ on  $\hat{\mc{B}}$.

Let $\mc{C}$ be the classifier class to which the output of  $\mc{O}_{kt, \alpha}$ is guaranteed to belong. With $n$ being the total number of distinct feature-vectors in the bags $\mc{B}$, $\Pi_{\mc{C}}(n)$ (as given in Theorem \ref{theorem:vcdim_growth_function}) is the number of possible $\{0,1\}$-assignments to $n$ points induced by classifiers in $\mc{C}$. Taking a union-bound over all of them, we obtain that with probability at most $\Pi_{\mc{C}}(n)\tn{exp}(-\alpha s/6)$ the output of $\mc{A}_2$ has accuracy at least $(1 - \eps)$ on $\mc{B}$. 

When $\mc{C}$ is unrestricted then $\Pi_{\mc{C}}(n) \leq 2^n$ and therefore $\Pi_{\mc{C}}(n)\tn{exp}(-\alpha s/6) \leq \delta$ is ensured by taking $s = O\left((n + \log(1/\delta))/\alpha\right)$. On the other hand if the VC dimension of $\mc{C}$ is at most $r$, then $\Pi_{\mc{C}}(n) \leq (en/r)^r$ (from Theorem \ref{theorem:vcdim_growth_function}) , and therefore taking $s = O\left(\frac{r}{\alpha}\log\left(\frac{n}{r}\right) + \log\left(\frac{1}{\delta}\right)\right)$ suffices.

\section{Experiments}\label{sec:experiments}

\begin{table*}[htb!]
\centering
\caption{Results on the Synthetic Datasets.}
\label{tab:table1-appendix}
\resizebox{\linewidth}{!}{
\footnotesize
\begin{tabular}{rrr|rrr|rrr}
\toprule
\multirow[c]{2}{*}{$q$} & \multirow[c]{2}{*}{$t\ $} & \multirow[c]{2}{*}{$s\ \ \ \ $} &  & Random Bags &  &   & Hard Bags &  \\
  &  &  & Large & Small& Test Instance &  Large & Small& Test Instance \\
\midrule
\multirow[c]{4}{*}{5} & \multirow[c]{2}{*}{10} & 5000 & $52.891 \pm 5.196$ & $85.357 \pm 3.085$ & $96.067 \pm 1.218$ & $32.629 \pm 3.439$ & $68.374 \pm 4.428$ & $91.120 \pm 1.978$ \\
 &  & 15000 & $72.295 \pm 5.275$ & $93.089 \pm 2.057$ & $97.840 \pm 0.829$ & $47.276 \pm 5.241$ & $81.802 \pm 3.789$ & $95.160 \pm 1.365$ \\
\cline{2-9}
 & \multirow[c]{2}{*}{50} & 5000 & $21.330 \pm 3.110$ & $85.513 \pm 3.434$ & $96.453 \pm 0.780$ & $12.789 \pm 2.192$ & $68.463 \pm 5.828$ & $91.427 \pm 1.785$ \\
 &  & 15000 & $32.890 \pm 5.032$ & $93.076 \pm 1.466$ & $97.867 \pm 0.626$ & $18.311 \pm 2.544$ & $82.562 \pm 3.637$ & $95.560 \pm 1.299$ \\
\cline{1-9} \cline{2-9}
\multirow[c]{4}{*}{15} & \multirow[c]{2}{*}{10} & 5000 & $21.792 \pm 3.189$ & $50.133 \pm 7.520$ & $93.037 \pm 1.674$ & $14.731 \pm 2.337$ & $31.600 \pm 5.138$ & $86.855 \pm 2.638$ \\
 &  & 15000 & $32.259 \pm 3.444$ & $68.733 \pm 4.334$ & $96.566 \pm 0.890$ & $17.115 \pm 1.501$ & $40.067 \pm 5.189$ & $89.939 \pm 1.921$ \\
\cline{2-9}
 & \multirow[c]{2}{*}{50} & 5000 & $8.674 \pm 1.537$ & $52.400 \pm 7.079$ & $93.778 \pm 2.060$ & $5.252 \pm 1.715$ & $34.000 \pm 5.438$ & $85.657 \pm 3.132$ \\
 &  & 15000 & $11.106 \pm 3.042$ & $67.467 \pm 4.389$ & $96.067 \pm 1.412$ & $6.409 \pm 1.457$ & $40.800 \pm 6.753$ & $91.677 \pm 2.336$ \\
\bottomrule
\end{tabular}
}
\end{table*}

\begin{table}[htb!]
\centering
\caption{Results on the Real Datasets.}
\label{tab:table2-appendix}
\resizebox{0.6\linewidth}{!}{
\footnotesize
\begin{tabular}{rrrrrr}%
\toprule
 $q$ & $t$ & $s$ & Large Bags & Small Bags & Test Instance \\
\midrule
\multicolumn{6}{c}{\textit{Heart}}\\
\multirow[c]{4}{*}{5} & \multirow[c]{2}{*}{10} & 2500 & $24.207 \pm 4.418$ & $55.407 \pm 8.419$ & $79.911 \pm 4.349$ \\
 &  & 10000 & $31.337 \pm 5.363$ & $65.333 \pm 8.516$ & $77.956 \pm 3.767$ \\
\cline{2-6}
 & \multirow[c]{2}{*}{50} & 2500 & $5.356 \pm 2.715$ & $47.407 \pm 8.172$ & $78.400 \pm 3.676$ \\
 &  & 10000 & $9.128 \pm 3.192$ & $59.556 \pm 8.021$ & $77.689 \pm 5.622$ \\
\cline{1-6} \cline{2-6}
\multirow[c]{4}{*}{15} & \multirow[c]{2}{*}{10} & 2500 & $12.950 \pm 7.030$ & $35.111 \pm 15.006$ & $71.378 \pm 7.870$ \\
 &  & 10000 & $20.539 \pm 8.041$ & $49.778 \pm 16.498$ & $69.156 \pm 7.089$ \\
\cline{2-6}
 & \multirow[c]{2}{*}{50} & 2500 & $0.803 \pm 1.521$ & $26.222 \pm 16.226$ & $73.867 \pm 5.829$ \\
 &  & 10000 & $1.946 \pm 2.143$ & $30.667 \pm 10.328$ & $72.178 \pm 6.852$ \\

\midrule

\multicolumn{6}{c}{\textit{Australian}}\\
\multirow[c]{4}{*}{5} & \multirow[c]{2}{*}{10} & 3500 & $24.956 \pm 3.709$ & $55.962 \pm 5.783$ & $84.275 \pm 2.626$ \\
 &  & 10000 & $29.774 \pm 2.600$ & $62.692 \pm 4.319$ & $84.039 \pm 1.999$ \\
\cline{2-6}
 & \multirow[c]{2}{*}{50} & 3500 & $5.454 \pm 4.127$ & $53.846 \pm 9.449$ & $82.039 \pm 3.015$ \\
 &  & 10000 & $9.303 \pm 2.806$ & $58.141 \pm 6.510$ & $82.431 \pm 2.837$ \\
\cline{1-6} \cline{2-6}
\multirow[c]{4}{*}{15} & \multirow[c]{2}{*}{10} & 3500 & $10.396 \pm 4.906$ & $28.190 \pm 8.072$ & $75.313 \pm 6.233$ \\
 &  & 10000 & $15.746 \pm 4.950$ & $37.524 \pm 10.792$ & $78.222 \pm 5.824$ \\
\cline{2-6}
 & \multirow[c]{2}{*}{50} & 3500 & $0.257 \pm 0.596$ & $24.190 \pm 7.233$ & $74.707 \pm 5.073$ \\
 &  & 10000 & $1.342 \pm 1.910$ & $30.095 \pm 8.215$ & $77.657 \pm 4.000$ \\

\midrule

\multicolumn{6}{c}{\textit{Adult}}\\
\multirow[c]{4}{*}{5} & \multirow[c]{2}{*}{10} & 10000 & $11.169 \pm 1.156$ & $41.418 \pm 2.684$ & $80.234 \pm 2.526$ \\
 &  & 80000 & $17.055 \pm 0.591$ & $47.873 \pm 0.716$ & $83.802 \pm 0.243$ \\
\cline{2-6}
 & \multirow[c]{2}{*}{50} & 10000 & $0.168 \pm 0.148$ & $34.396 \pm 2.668$ & $75.651 \pm 3.222$ \\
 &  & 80000 & $2.161 \pm 0.306$ & $46.835 \pm 1.060$ & $83.111 \pm 0.831$ \\
\cline{1-6} \cline{2-6}
\multirow[c]{4}{*}{15} & \multirow[c]{2}{*}{10} & 10000 & $1.515 \pm 0.853$ & $13.000 \pm 1.970$ & $76.005 \pm 3.249$ \\
 &  & 80000 & $5.801 \pm 0.760$ & $22.878 \pm 1.316$ & $83.461 \pm 0.822$ \\
\cline{2-6}
 & \multirow[c]{2}{*}{50} & 10000 & $0.001 \pm 0.003$ & $8.797 \pm 5.715$ & $75.077 \pm 2.638$ \\
 &  & 80000 & $0.044 \pm 0.036$ & $21.498 \pm 0.667$ & $82.185 \pm 0.908$ \\
\bottomrule
\end{tabular}
}
\end{table}

In our experiments, we generate a collection of small $q$-sized bags as training data using fully supervised datasets. We use a fixed value of $q \in \{5, 15\}$.

\noindent
\textbf{Synthetic Datasets.} In this case we experiment in the realizable setting for which we select a random linear classifier $f^*$ passing though the origin to provide $\{0,1\}$-labels to the feature-vectors. For a given bag-size $q \in \{5, 15\}$, we  generate two types of bag collections as follows:
\begin{enumerate}[leftmargin=*,noitemsep,nolistsep]
    \item \textit{Random}: In this case each $q$-sized bag is created by randomly sampling points uniformly from the unit sphere as its constituent feature vectors.
    \item \textit{Hard Bags}: For these bags we first randomly construct pairs of points on the unit-sphere which are either (i) very close but have different labels under $f^*$, or (ii) nearly antipodal but have the same label. Each bag consists of several such randomly constructed pairs and one random point (since $q$ is odd).
\end{enumerate}
In both the above cases, the aggregate label of a bag is the sum of the labels of its feature-vectors given by $f^*$.We also have a test-set of labeled feature-vectors whose distribution is given by sampling each u.a.r. from a random training bag.

\noindent
\textbf{Real Datasets.} We use the following supervised UCI datasets: \textit{Heart} (303 instances, \citep{misc_heart_disease_45}), \textit{Australian} (690 instances, \citep{misc_statlog_(australian_credit_approval)_143}) and \textit{Adult} (48842 instances, \citep{misc_adult_2}) which have previously been used by \cite{PNCR14} to evaluate LLP methods. The feature-vector labels are available and the bags are created by partitioning the training-set into 
-sized bags. The test-set is given by a random subset of 15\% of the dataset.

\noindent
\textbf{Applying Algorithm $\mc{A}_2$.} For each collection of training bags, and an appropriate choice of $t$ and $s$ (see Figure \ref{algo:A2})  we create a collection of $s$ large bags by sampling each iid from the distribution $\ol{D}$ given in Figure \ref{algo:DistnDbar}.

\noindent
\textbf{Model Training.} We train a linear model $g(\bx)$ with a sigmoid activation function on the large bags using bag-level mse loss between the aggregate label of a bag and its aggregate prediction. In particular, for a large bag $\ol{B}$ and aggregate label $\ol{\sigma}$ the contribution to the loss is $\left(\ol{\sigma} - \sum_{\bx \in \ol{B}}g(\bx) \right)^2$. and the total loss is the sum over the 
 large bags in collection. The optimization is done using a mini-batch training with 512 bags in each mini-batch. The learning rate is 1e-2 with SGD optimizer for all experiments, and the model is trained till it reaches convergence on the instance-level test set.

\noindent
\textbf{Results.} Tables \ref{tab:table1-appendix} and \ref{tab:table2-appendix} have the experimental results for the synthetic, Heart, Australian and Adult datasets respectively. For each setting of $q$, $t$ and $s$, we report the mean accuracy and standard deviation on the training set for both large bags and their constituent small bags, along with the accuracy on test instances, averaged over $15$ runs. The main takeaways from the experimental results are:
\begin{enumerate}[leftmargin=*,noitemsep,nolistsep]
\item In all experiments, even with low accuracy on large bags we obtain classifiers with high accuracy on the constituent small bags and even higher accuracy on the instance-level test set. For example, on synthetic random bags with $q=5, t=50$ and $s=5000$, an accuracy of just $21.3\%$ on large bags yields an accuracy of $85.5\%$ on small bags and $96.4\%$ on the test set. On the Adult dataset, with $q=15, t=50$ and $s=80000$, with accuracy of just $0.044\%$ on large bags, we obtain a classifier with accuracy of $21.5\%$ on small bags and $82.2\%$ on the test set.
\item For a given  $q$ and $t$, increasing the number of large bags $s$  improves performance across the board, consistent with our theoretical bounds.
\item The bag-level performance scores are noticeably lower on the hard bags case as compared to the random bags case, even though both are from the realizable setting.
\item Accuracy scores on large bags decrease with increasing $q$ or $t$. This is understandable since this results in increased size of larger bags, making them more difficult to satisfy.
\end{enumerate}
The above observations, especially points 1 and 2, demonstrate that Algorithm $\mc{A}_2$ does indeed provide a way to use weak classifiers on large bags to obtain strong classifiers on small bags, which in turn are strong classifiers at the instance-level. The scalability of our techniques is also validated by the experiments on the substantially sized Adult dataset. Each of these experiments on a standard GPU/CPU took less than 12 hrs, and most completed within an hour\footnote{The code for these experiments will be made public along with the final version of this paper.}. For each dataset, the small bags were fixed, and large bags were sampled for each repeated run of the experiment. For the synthetic, Heart, and Australian datasets, the model was trained for 160 epochs, while for the Adult dataset, it was trained for 60 epochs. Each experiment was run on a single NVIDIA A100 40GB GPU and 2x Intel Broadwell 22 cores 44 threads CPU.

\section{Conclusion} \label{sec:conclusion}
In conclusion, our study is the first to demonstrate the impossibility of boosting weak classifiers to a strong classifier in the LLP and MIL settings. For LLP our work rules out boosting using weak classifiers of any accuracy $ < 1$, while for MIL the possibility of boosting weak classifiers with accuracy $< 2/3$ is eliminated. Complementing these findings in the LLP context, we propose an algorithm that converts a weak classifier for large bags into a strong classifier for an input collection of small bags. The algorithm constructs unions of constantly many small bags to achieve error amplification. A more efficient sampling based version of the same provides high probability guarantees, which we also experimentally validate on three real and two synthetic datasets.

\noindent
{\bf Limitations.} While this work has advanced the understanding of weak to strong learning in aggregate label settings,  our analysis does not rule boosting for MIL using weak classifiers with accuracy in $[2/3, 1)$. A related question remains on how to effectively obtain a strong classifier in the MIL setting using weak classifiers on larger bags.

\bibliographystyle{plainnat}
\bibliography{references}

\begin{thebibliography}{56}
\providecommand{\natexlab}[1]{#1}
\providecommand{\url}[1]{\texttt{#1}}
\expandafter\ifx\csname urlstyle\endcsname\relax
  \providecommand{\doi}[1]{doi: #1}\else
  \providecommand{\doi}{doi: \begingroup \urlstyle{rm}\Url}\fi

\bibitem[Anthony and Bartlett(1999)]{anthony_bartlett_1999neural}
M.~Anthony and P.L. Bartlett.
\newblock \emph{Neural Network Learning: Theoretical Foundations}.
\newblock Cambridge University Press, 1999.
\newblock ISBN 9780521573535.
\newblock URL \url{https://books.google.co.in/books?id=UH6XRoEQ4h8C}.

\bibitem[Auer and Ortner(2004)]{MILBoost}
Peter Auer and Ronald Ortner.
\newblock A boosting approach to multiple instance learning.
\newblock In \emph{Machine Learning: ECML 2004}, pages 63--74, 2004.

\bibitem[Becker and Kohavi(1996)]{misc_adult_2}
Barry Becker and Ronny Kohavi.
\newblock {Adult}.
\newblock UCI Machine Learning Repository, 1996.
\newblock {DOI}: https://doi.org/10.24432/C5XW20.

\bibitem[Blum and Kalai(1998)]{blum1998note}
Avrim Blum and Adam Kalai.
\newblock A note on learning from multiple-instance examples.
\newblock \emph{Machine learning}, 30:\penalty0 23--29, 1998.

\bibitem[Brahmbhatt et~al.(2023)Brahmbhatt, Saket, and
  Raghuveer]{brahmbhatt2023pac}
Anand~Paresh Brahmbhatt, Rishi Saket, and Aravindan Raghuveer.
\newblock {PAC} learning linear thresholds from label proportions.
\newblock In \emph{Proc. NeurIPS}, 2023.
\newblock URL \url{https://openreview.net/forum?id=5Gw9YkJkFF}.

\bibitem[Busa-Fekete et~al.(2023)Busa-Fekete, Choi, Dick, Gentile, and
  medina]{busafekete2023easy}
Robert~Istvan Busa-Fekete, Heejin Choi, Travis Dick, Claudio Gentile, and
  Andres~Munoz medina.
\newblock Easy learning from label proportions.
\newblock \emph{arXiv}, 2023.
\newblock URL \url{https://arxiv.org/abs/2302.03115}.

\bibitem[Chen et~al.(2004)Chen, Huang, and Ramakrishnan]{CHR}
L.~Chen, Z.~Huang, and R.~Ramakrishnan.
\newblock Cost-based labeling of groups of mass spectra.
\newblock In \emph{Proc. ACM SIGMOD International Conference on Management of
  Data}, pages 167--178, 2004.

\bibitem[Chen et~al.(2023)Chen, Fu, Karbasi, and Mirrokni]{chen2023learning}
Lin Chen, Thomas Fu, Amin Karbasi, and Vahab Mirrokni.
\newblock Learning from aggregated data: Curated bags versus random bags.
\newblock \emph{arXiv}, 2023.
\newblock URL \url{https://arxiv.org/abs/2305.09557}.

\bibitem[Chen et~al.(2009)Chen, Liu, Qian, and Zhang]{CLQZ09}
S.~Chen, B.~Liu, M.~Qian, and C.~Zhang.
\newblock Kernel k-means based framework for aggregate outputs classification.
\newblock In \emph{Proc. {ICDM}}, pages 356--361, 2009.

\bibitem[Chen and Guestrin(2016)]{xgboost}
Tianqi Chen and Carlos Guestrin.
\newblock Xgboost: A scalable tree boosting system.
\newblock In \emph{Proceedings of the 22nd ACM SIGKDD International Conference
  on Knowledge Discovery and Data Mining}, KDD '16, page 785–794, New York,
  NY, USA, 2016. Association for Computing Machinery.
\newblock ISBN 9781450342322.
\newblock \doi{10.1145/2939672.2939785}.
\newblock URL \url{https://doi.org/10.1145/2939672.2939785}.

\bibitem[de~Freitas and K{\"{u}}ck(2005)]{FK05}
N.~de~Freitas and H.~K{\"{u}}ck.
\newblock Learning about individuals from group statistics.
\newblock In \emph{Proc. {UAI}}, pages 332--339, 2005.

\bibitem[Dery et~al.(2017)Dery, Nachman, Rubbo, and Schwartzman]{DNRS}
L.~M. Dery, B.~Nachman, F.~Rubbo, and A.~Schwartzman.
\newblock Weakly supervised classification in high energy physics.
\newblock \emph{Journal of High Energy Physics}, 2017\penalty0 (5):\penalty0
  1--11, 2017.

\bibitem[Dietterich et~al.(1997)Dietterich, Lathrop, and
  Lozano{-}P{\'{e}}rez]{DLL97}
Thomas~G. Dietterich, Richard~H. Lathrop, and Tom{\'{a}}s Lozano{-}P{\'{e}}rez.
\newblock Solving the multiple instance problem with axis-parallel rectangles.
\newblock \emph{Artif. Intell.}, 89\penalty0 (1-2):\penalty0 31--71, 1997.

\bibitem[Dulac{-}Arnold et~al.(2019)Dulac{-}Arnold, Zeghidour, Cuturi, Beyer,
  and Vert]{DZCBV19}
G.~Dulac{-}Arnold, N.~Zeghidour, M.~Cuturi, L.~Beyer, and J.~P. Vert.
\newblock Deep multi-class learning from label proportions.
\newblock \emph{CoRR}, abs/1905.12909, 2019.
\newblock URL \url{http://arxiv.org/abs/1905.12909}.

\bibitem[Erd{\"o}s(1945)]{littlewood_offord_Erds1945OnAL}
Paul Erd{\"o}s.
\newblock On a lemma of littlewood and offord.
\newblock \emph{Bulletin of the American Mathematical Society}, 51:\penalty0
  898--902, 1945.
\newblock URL \url{https://api.semanticscholar.org/CorpusID:122046405}.

\bibitem[Feige and Schechtman(2002)]{Feige2002}
Uriel Feige and Gideon Schechtman.
\newblock On the optimality of the random hyperplane rounding technique for max
  cut.
\newblock \emph{Random Structures \& Algorithms}, 20\penalty0 (3):\penalty0
  403--440, 2002.
\newblock \doi{https://doi.org/10.1002/rsa.10036}.
\newblock URL \url{https://onlinelibrary.wiley.com/doi/abs/10.1002/rsa.10036}.

\bibitem[Freund(1990)]{Freund}
Yoav Freund.
\newblock Boosting a weak learning algorithm by majority.
\newblock In \emph{Proc. {COLT}}, pages 202--216, 1990.

\bibitem[Freund(2001)]{brown_boost_Freund2001}
Yoav Freund.
\newblock An adaptive version of the boost by majority algorithm.
\newblock \emph{Machine Learning}, 43\penalty0 (3):\penalty0 293--318, Jun
  2001.
\newblock ISSN 1573-0565.
\newblock \doi{10.1023/A:1010852229904}.
\newblock URL \url{https://doi.org/10.1023/A:1010852229904}.

\bibitem[Freund and Schapire(1995)]{AdaBoost}
Yoav Freund and Robert~E. Schapire.
\newblock A decision-theoretic generalization of on-line learning and an
  application to boosting.
\newblock In \emph{Proc. EuroCOLT}, volume 904 of \emph{Lecture Notes in
  Computer Science}, pages 23--37. Springer, 1995.
\newblock URL \url{https://doi.org/10.1007/3-540-59119-2\_166}.

\bibitem[Friedman et~al.(2000)Friedman, Hastie, and Tibshirani]{LogitBoost}
Jerome Friedman, Trevor Hastie, and Robert Tibshirani.
\newblock {Additive logistic regression: a statistical view of boosting}.
\newblock \emph{The Annals of Statistics}, 28\penalty0 (2):\penalty0 337 --
  407, 2000.

\bibitem[Hern{\'{a}}ndez{-}Gonz{\'{a}}lez
  et~al.(2013)Hern{\'{a}}ndez{-}Gonz{\'{a}}lez, Inza, and Lozano]{HIL13}
J.~Hern{\'{a}}ndez{-}Gonz{\'{a}}lez, I.~Inza, and J.~A. Lozano.
\newblock Learning bayesian network classifiers from label proportions.
\newblock \emph{Pattern Recognit.}, 46\penalty0 (12):\penalty0 3425--3440,
  2013.

\bibitem[Hernández-González et~al.(2018)Hernández-González, Inza,
  Crisol-Ortíz, Guembe, Iñarra, and Lozano]{hernandez2018}
J.~Hernández-González, I.~Inza, L.~Crisol-Ortíz, M.~A. Guembe, M.~J.
  Iñarra, and J.~A. Lozano.
\newblock Fitting the data from embryo implantation prediction: Learning from
  label proportions.
\newblock \emph{Statistical methods in medical research}, 27\penalty0
  (4):\penalty0 1056--1066, 2018.

\bibitem[Janosi and Detrano(1988)]{misc_heart_disease_45}
Steinbrunn-William Pfisterer~Matthias Janosi, Andras and Robert Detrano.
\newblock {Heart Disease}.
\newblock UCI Machine Learning Repository, 1988.
\newblock {DOI}: https://doi.org/10.24432/C52P4X.

\bibitem[Kotzias et~al.(2015)Kotzias, Denil, de~Freitas, and Smyth]{KDFS15}
D.~Kotzias, M.~Denil, N.~de~Freitas, and P.~Smyth.
\newblock From group to individual labels using deep features.
\newblock In \emph{Proc. SIGKDD}, pages 597--606, 2015.

\bibitem[Lai et~al.(2023)Lai, Xiao, and Liu]{two_view_llp_boost_Lai2023}
Jiantao Lai, Yanshan Xiao, and Bo~Liu.
\newblock Boost two-view learning-based method for label proportions problem.
\newblock \emph{Applied Intelligence}, 53\penalty0 (19):\penalty0 21984--22001,
  Oct 2023.
\newblock ISSN 1573-7497.
\newblock \doi{10.1007/s10489-023-04643-z}.
\newblock URL \url{https://doi.org/10.1007/s10489-023-04643-z}.

\bibitem[Liu et~al.(2019)Liu, Wang, Qi, Tian, and Shi]{LWQTS19}
J.~Liu, B.~Wang, Z.~Qi, Y.~Tian, and Y.~Shi.
\newblock Learning from label proportions with generative adversarial networks.
\newblock In \emph{Proc. {NeurIPS}}, pages 7167--7177, 2019.

\bibitem[Lozano-Pérez and Yang(2000)]{LY00}
T.~Lozano-Pérez and C.~Yang.
\newblock Image database retrieval with multiple-instance learning techniques.
\newblock In \emph{Proc. ICDE}, page 233, 2000.

\bibitem[Maron(1998)]{M98}
O.~Maron.
\newblock \emph{Learning from ambiguity}.
\newblock PhD thesis, Massassachusetts Institute of Technology, 1998.

\bibitem[Maron and Lozano-P\'{e}rez(1997)]{ML97}
Oded Maron and Tom\'{a}s Lozano-P\'{e}rez.
\newblock A framework for multiple-instance learning.
\newblock NIPS'97, page 570–576, 1997.

\bibitem[Mason et~al.(1999)Mason, Baxter, Bartlett, and Frean]{AnyBoost}
Llew Mason, Jonathan Baxter, Peter~L. Bartlett, and Marcus~R. Frean.
\newblock Boosting algorithms as gradient descent.
\newblock In \emph{Proc. {NIPS}}, pages 512--518, 1999.

\bibitem[Musicant et~al.(2007)Musicant, Christensen, and Olson]{MCO07}
D.~R. Musicant, J.~M. Christensen, and J.~F. Olson.
\newblock Supervised learning by training on aggregate outputs.
\newblock In \emph{Proc. {ICDM}}, pages 252--261. {IEEE} Computer Society,
  2007.

\bibitem[Nandy et~al.(2022)Nandy, Saket, Jain, Chauhan, Ravindran, and
  Raghuveer]{NSJCRR22}
J.~Nandy, R.~Saket, P.~Jain, J.~Chauhan, B.~Ravindran, and A.~Raghuveer.
\newblock Domain-agnostic contrastive representations for learning from label
  proportions.
\newblock In \emph{Proc. CIKM}, pages 1542--1551, 2022.

\bibitem[O'Brien et~al.(2022)O'Brien, Thiagarajan, Das, Barreto, Verma, Hsu,
  Neufeld, and Hunt]{Obrien}
Conor O'Brien, Arvind Thiagarajan, Sourav Das, Rafael Barreto, Chetan Verma,
  Tim Hsu, James Neufeld, and Jonathan~J. Hunt.
\newblock Challenges and approaches to privacy preserving post-click conversion
  prediction.
\newblock \emph{CoRR}, abs/2201.12666, 2022.
\newblock URL \url{https://arxiv.org/abs/2201.12666}.

\bibitem[Patrini et~al.(2014)Patrini, Nock, Caetano, and Rivera]{PNCR14}
G.~Patrini, R.~Nock, T.~S. Caetano, and P.~Rivera.
\newblock (almost) no label no cry.
\newblock In \emph{Proc. Advances in Neural Information Processing Systems},
  pages 190--198, 2014.

\bibitem[Qi et~al.(2018)Qi, Meng, Tian, Niu, Shi, and Zhang]{AdaBoostLLP}
Zhiquan Qi, Fan Meng, Yingjie Tian, Lingfeng Niu, Yong Shi, and Peng Zhang.
\newblock {Adaboost-LLP}: A boosting method for learning with label
  proportions.
\newblock \emph{IEEE Transactions on Neural Networks and Learning Systems},
  29\penalty0 (8):\penalty0 3548--3559, 2018.

\bibitem[Quadrianto et~al.(2009)Quadrianto, Smola, Caetano, and Le]{QSCL09}
N.~Quadrianto, A.~J. Smola, T.~S. Caetano, and Q.~V. Le.
\newblock Estimating labels from label proportions.
\newblock \emph{J. Mach. Learn. Res.}, 10:\penalty0 2349--2374, 2009.

\bibitem[Quinlan()]{misc_statlog_(australian_credit_approval)_143}
Ross Quinlan.
\newblock {Statlog (Australian Credit Approval)}.
\newblock UCI Machine Learning Repository.
\newblock {DOI}: https://doi.org/10.24432/C59012.

\bibitem[Ramon and De~Raedt(2000)]{ramon2000multi}
Jan Ramon and Luc De~Raedt.
\newblock Multi instance neural networks.
\newblock In \emph{Proceedings of the ICML-2000 workshop on attribute-value and
  relational learning}, pages 53--60, 2000.

\bibitem[Ray and Craven(2005)]{RC05}
Soumya Ray and Mark Craven.
\newblock Supervised versus multiple instance learning: an empirical
  comparison.
\newblock In \emph{Proc. ICML}, page 697–704, 2005.

\bibitem[Rueping(2010)]{R10}
S.~Rueping.
\newblock {SVM} classifier estimation from group probabilities.
\newblock In \emph{Proc. ICML}, pages 911--918, 2010.

\bibitem[Sabato and Tishby(2012)]{ST12}
Sivan Sabato and Naftali Tishby.
\newblock Multi-instance learning with any hypothesis class.
\newblock \emph{Journal of Machine Learning Research}, 13\penalty0
  (97):\penalty0 2999--3039, 2012.
\newblock URL \url{http://jmlr.org/papers/v13/sabato12a.html}.

\bibitem[Saket(2021)]{Saket21}
R.~Saket.
\newblock Learnability of linear thresholds from label proportions.
\newblock In \emph{Proc. NeurIPS}, 2021.
\newblock URL \url{https://openreview.net/forum?id=5BnaKeEwuYk}.

\bibitem[Saket(2022)]{Saket22}
R.~Saket.
\newblock Algorithms and hardness for learning linear thresholds from label
  proportions.
\newblock In \emph{Proc. NeurIPS}, 2022.
\newblock URL \url{https://openreview.net/forum?id=4LZo68TuF-4}.

\bibitem[Saket et~al.(2022)Saket, Raghuveer, and Ravindran]{SRR}
Rishi Saket, Aravindan Raghuveer, and Balaraman Ravindran.
\newblock On combining bags to better learn from label proportions.
\newblock In \emph{{AISTATS}}, volume 151 of \emph{Proceedings of Machine
  Learning Research}, pages 5913--5927. {PMLR}, 2022.
\newblock URL \url{https://proceedings.mlr.press/v151/saket22a.html}.

\bibitem[Schapire(1989)]{Schapire}
Robert~E. Schapire.
\newblock The strength of weak learnability (extended abstract).
\newblock In \emph{Proc. {FOCS}}, pages 28--33, 1989.

\bibitem[Schapire and Freund(2012)]{FSBook}
Robert~E. Schapire and Yoav Freund.
\newblock \emph{Boosting: Foundations and Algorithms}.
\newblock The MIT Press, 2012.
\newblock ISBN 0262017180.

\bibitem[Sikka et~al.(2013)Sikka, Dhall, and Bartlett]{SDB13}
Karan Sikka, Abhinav Dhall, and Marian Bartlett.
\newblock Weakly supervised pain localization using multiple instance learning.
\newblock In \emph{2013 10th {IEEE} International Conference and Workshops on
  Automatic Face and Gesture Recognition (FG)}, pages 1--8, 2013.

\bibitem[Stolpe and Morik(2011)]{SM11}
M.~Stolpe and K.~Morik.
\newblock Learning from label proportions by optimizing cluster model
  selection.
\newblock In \emph{{ECML} {PKDD} Proceedings, Part {III}}, volume 6913, pages
  349--364. Springer, 2011.

\bibitem[Viola et~al.(2005)Viola, Platt, and Zhang]{ViolaPZ05}
Paul Viola, John~C. Platt, and Cha Zhang.
\newblock Multiple instance boosting for object detection.
\newblock In \emph{NIPS}, page 1417–1424, 2005.

\bibitem[Warmuth et~al.(2008)Warmuth, Glocer, and Vishwanathan]{ent_lp_boost}
Manfred~K. Warmuth, Karen~A. Glocer, and S.~V.~N. Vishwanathan.
\newblock Entropy regularized lpboost.
\newblock In Yoav Freund, L{\'a}szl{\'o} Gy{\"o}rfi, Gy{\"o}rgy Tur{\'a}n, and
  Thomas Zeugmann, editors, \emph{Algorithmic Learning Theory}, pages 256--271,
  Berlin, Heidelberg, 2008. Springer Berlin Heidelberg.
\newblock ISBN 978-3-540-87987-9.

\bibitem[Wu et~al.(2015)Wu, Yu, Huang, and Yu]{WYY15}
J.~Wu, Yinan Yu, Chang Huang, and Kai Yu.
\newblock Deep multiple instance learning for image classification and
  auto-annotation.
\newblock In \emph{Proc. {CVPR}}, pages 3460--3469, 2015.

\bibitem[Yu et~al.(2013)Yu, Liu, Kumar, Jebara, and Chang]{YLKJC13}
F.~X. Yu, D.~Liu, S.~Kumar, T.~Jebara, and S.~F. Chang.
\newblock $\propto${SVM} for learning with label proportions.
\newblock In \emph{Proc. ICML}, volume~28, pages 504--512, 2013.

\bibitem[Yu et~al.(2014)Yu, Choromanski, Kumar, Jebara, and Chang]{YCKJC14}
F.~X. Yu, K.~Choromanski, S.~Kumar, T.~Jebara, and S.~F. Chang.
\newblock On learning from label proportions.
\newblock \emph{CoRR}, abs/1402.5902, 2014.
\newblock URL \url{http://arxiv.org/abs/1402.5902}.

\bibitem[Zhang et~al.(2005)Zhang, Platt, and Viola]{ZPV05}
Cha Zhang, John Platt, and Paul Viola.
\newblock Multiple instance boosting for object detection.
\newblock In \emph{Advances in Neural Information Processing Systems},
  volume~18. MIT Press, 2005.

\bibitem[Zhang et~al.(2022)Zhang, Wang, and Scott]{ZWS22}
J.~Zhang, Y.~Wang, and C.~Scott.
\newblock Learning from label proportions by learning with label noise.
\newblock In \emph{Proc. NeurIPS}, 2022.

\bibitem[Zhang and Goldman(2001)]{ZG01}
Qi~Zhang and Sally Goldman.
\newblock Em-dd: An improved multiple-instance learning technique.
\newblock \emph{Advances in neural information processing systems}, 14, 2001.

\end{thebibliography}
\appendix

\section{Trivial Accuracy in the LLP and MIL} \label{app:trivial_performance}

First consider the LLP bags $\mc{B}$ from Theorem \ref{thm:LLP-imposssibility}, each bag is of size $2$ with aggregate label $1$ i.e., it is satisfied if exactly one of its feature-vectors is labeled $1$. Now consider just one bag from $\mc{B}$. This bag is not satisfied by the constant $0$ or constant $1$ classifier. On the other hand the expected accuracy of random labeling is $1/2$, and therefore ${\sf Trv}_{\sf LLP}(\mc{B}) = 1/2$.

Next, let $\mc{B}$ be the MIL bags from Theorem \ref{thm:MIL-imposssibility}. These are bags of size $2$ each and some of them have aggregate label $0$ and some have aggregate label $1$. Consider just two bags one with aggregate label $0$ and the other with aggregate label $1$. Now, the constant $a$ labeling satisfies the bag with aggregate label $a$ and does not satisfy the bag with aggregate label $(1-a)$, for $a \in \{0,1\}$. On the other hand the random labeling satisfies the $0$ aggregate label bag with probability $1/4$ and the bag with aggregate label $1$ with probability $3/4$. Thus, the expected number of bags satisfied is the random labeling is $1$. Therefore,  ${\sf Trv}_{\sf MIL}(\mc{B}) = 1/2$.

\section{Weighted  bags to unweighted bags} 
\label{app:wtdtounwtd}
\begin{figure}[!htb]
\begin{mdframed}
\small
\textbf{Input:} : Bags $\mc{B}_w = (B_i, w_i)_{i=1}^m$, $T$.\\
\\
\textbf{Steps:}
\\ \textbf{1.} Normalize the weight with a factor $Z$ such that $\sum_{i=1}^m w_i = m$.
\\ \textbf{2.} Define $\mc{B}$ to be the unweighted collection of bags and initialize it to $\emptyset$.
\\ \textbf{3.} for  $i \in [m]$: \\
\hspace*{2em} \textbf{3.1} Define $n_i =  \lceil{w_i(T-1)}\rceil$.\\
\hspace*{2em} \textbf{3.2} Add $n_i$ copies of $B_i$ to $\mc{B}$.\\
\\
\textbf{Output:} Output $\mc{B}$.
\end{mdframed}
\caption{Weighted to unweighted collection of bags}\label{fig:wtdtounwtd}
\end{figure}
The algorithm to convert a weighted collection of bags to an unweighted collection is given in Fig. \ref{fig:wtdtounwtd}. First, observe that $|\mc{B}| = \sum_{i=1}^m \lceil{w_i(T-1)}\rceil \leq \sum_{i=1}^m (w_i(T-1)+1) \leq (T-1)m + m = Tm$, where we use $\sum_{i=1}^m w_i = m$. On the other hand, $|\mc{B}| =  \sum_{i=1}^m \lceil{w_i(T-1)}\rceil \geq (T-1)m$. 

Now, to see that the error in accuracy is at most $O(1/T)$, observe that for any subset $I \subseteq [m]$,  $\sum_{i\in I} {w_i(T-1)} \leq \sum_{i\in I} \lceil{w_i(T-1)}\rceil \leq \sum_{i\in I} {w_i(T-1)} + |I|$. Therefore, the normalized error in the weight corresponding to $I$ is at most $|I|/((T-1)m) \leq m/((T-1)m) \leq 1/(T-1) = O(1/T)$ for $T > 1$. 

\section{Probabilities for the support of $\ol{D}$}\label{sec:suppD}
In Step 2 of Figure \ref{algo:DistnDbar}, the for a fixed configuration $\{\mc{Q}_i\}_{i=1}^t$ with $r : |\{ i \in [t]\,\mid\, \mc{Q}_i = \star\}|$, its probability under $\ol{D}$ is $\frac{m^r}{m^t}\frac{1}{2^r}$, since the number of choices for the $\star$-coordinates is $m^r$, while the total number of choices is $m^t$. Further, with $(1/2)^t$ probability we have the specific choices of the $r$ coordinates with $\star$ in Step 2. Iterating over all possible configurations   $\{\mc{Q}_i\}_{i=1}^t$ and assigning theur probabilities to the resultant $(\ol{B},\ol{\sigma})$ in Step 3, yields the support of $\ol{D}$ along with their probabilities.

\end{document}